\documentclass[12pt]{article}

\usepackage[misc]{ifsym}
\newcommand{\corr}{(\Letter)}
\usepackage[utf8]{inputenc}
\usepackage[T1]{fontenc}
\usepackage[english]{babel}
\usepackage[top=1in, bottom=1in, left=1in, right=1in]{geometry}
\usepackage{setspace}
\usepackage{fancyhdr}
\newcommand{\runningheads}[2]{
    \setlength{\headheight}{15pt}
    \addtolength{\topmargin}{-2pt}
    \pagestyle{fancy}  
    \fancyhead{}  
    \fancyhead[L]{\small #1}
    \fancyhead[R]{\small #2}
    \renewcommand{\headrulewidth}{0.4pt}
}

\usepackage{float}
\usepackage{graphicx}
\usepackage{subcaption}
\usepackage{booktabs}
\usepackage{makecell}
\usepackage{multirow}
\usepackage{threeparttable}

\usepackage[linesnumbered,ruled,vlined]{algorithm2e}
\SetCommentSty{textnormal}
\usepackage{enumitem}
\usepackage{pdfpages}

\usepackage{amsmath}
\usepackage{amssymb}
\usepackage{amsthm}

\theoremstyle{definition}
\newtheorem{proposition}{Proposition}
\newtheorem{theorem}{Theorem}

\newtheorem{remark}{Remark}
\newtheorem{assumption}{Assumption}

\usepackage{authblk}

\usepackage{natbib}
\usepackage{xcolor}
\definecolor{darkblue}{rgb}{0, 0, 0.75}
\definecolor{darkgreen}{rgb}{0, 0.75, 0}
\usepackage[colorlinks=true, linkcolor=black, citecolor=darkgreen, urlcolor=darkblue]{hyperref}

\newenvironment{foldable}{}{}

\newcommand*{\eg}{e.g.}

\newcommand{\meanstd}[2]{#1$_{\pm{#2}}$}
\newcommand{\meanstdbf}[2]{\textbf{#1}$_{\pm{#2}}$}
\newcommand*{\bltrg}{$^{\blacktriangle}$}
\newcommand{\dg}{\rlap{$^\dagger$}}

\title{TAKE: Trajectory-Aware Knowledge Estimation \\ for Text Dataset Distillation}
\author[]{Tri-Nhan Vo\corr}
\author[]{Dang Nguyen}
\author[]{Sunil Gupta}
\affil[]{Deakin Applied AI Initiative, Deakin University, Australia}
\affil[]{\texttt{\{s223032975, d.nguyen, sunil.gupta\}@deakin.edu.au}}
\runningheads{T.-N. Vo et al.}{TAKE: Trajectory-Aware Knowledge Estimation}
\date{\small{March 2026}}

\begin{document}

\maketitle

\begin{abstract} \label{sec:00-abstract}
  Large-scale text corpora have become a quiet bottleneck in modern NLP, not just in storage, but in the accumulated cost of training, fine-tuning, and continual learning.
  We propose a text dataset distillation framework that reduces corpora to as little as 0.1\% of their original size while preserving downstream task fidelity.
  We approach distillation through the lens of influence functions, which quantify each sample's contribution to the downstream objective, a natural and principled basis for selection.
  We introduce \textbf{Trajectory-Aware Knowledge Estimation} (\textbf{TAKE}), which convolves the knowledge-based influence along the training trajectory into a single per-sample knowledge score, capturing informative samples.
  These scores serve as sample weights within a discrete Optimal Transport objective, guiding prototype selection from a synthetically generated candidate pool.
  We evaluate TAKE on downstream accuracy across text classification and natural language inference tasks at extreme compression (0.1\% or 20 samples/class), showing that data efficiency is achievable without sacrificing task fidelity.
  The approach is theoretically grounded, with broader implications for coreset construction and data-centric AI.
  We release our source code at \url{https://github.com/votrinhan88/take}.

\end{abstract}

\section{Introduction} \label{sec:01-introduction}
\begin{foldable} 
  The progress of large language models depends on large-scale supervised corpora, yet the cost of curating and training on such data compounds across fine-tuning cycles, hyperparameter sweeps, and continual learning.
  The real burden is not any single run but the cumulative expenditure across the full development lifecycle.
  Corpus size, not model size, is often the binding constraint.
  Reducing it is therefore a systemic priority with direct implications for cost, carbon footprint, and data governance.
\end{foldable}

\begin{foldable} 
  Dataset distillation (DD), first proposed by~\cite{wang2018dataset}, replaces a large dataset $\mathcal{D}$ with a smaller surrogate $\tilde{\mathcal{D}} \ll \mathcal{D}$ such that a model trained on $\tilde{\mathcal{D}}$ performs comparably to one trained on $\mathcal{D}$.
  Put another way: what is the smallest textbook that teaches the same exam?
  This places DD squarely within data-centric AI, complementing model compression by operating on the data rather than the model itself.
  Early DD methods work in image space, optimizing synthetic pixels via meta-learning, gradient or trajectory matching~\cite{wang2018dataset,zhao2020dataset,cazenavette2022dataset}.
  Text breaks these assumptions: non-differentiable token decoding blocks gradient-based synthesis, and embedding spaces are tightly coupled to specific architectures.
  Adapting DD to text has therefore required fundamentally different strategies.
\end{foldable}

\begin{foldable} 
  Existing text DD methods~\cite{sucholutsky2021soft,li2021data,maekawa2023dataset,maekawa2024dilm,tao2024textual} fall short in at least one of the following: at extreme compression, uniform weighting exhausts the budget on uninformative samples; their objectives lack task alignment or global optimization; and their outputs are embeddings, neither auditable nor transferable.
  We survey this landscape in \S\ref{sec:02-related-textdd} and show that no prior method simultaneously addresses all of the above.
  These gaps motivate three desiderata for TAKE:{
    (1)~\emph{knowledge-based weighting} --- weighting samples by their task-aligned downstream contribution, directing the budget toward the most informative ones;
    (2)~a \emph{dataset-level distillation objective} --- optimizing over the full training distribution rather than batch-by-batch, ensuring global coverage;
    and (3)~\emph{human-readable output} --- a distilled corpus that is directly inspectable, auditable, and transferable across architectures.
  }
\end{foldable}

\begin{foldable} 
  We propose \textbf{Trajectory-Aware Knowledge Estimation} (\textbf{TAKE}), the first text DD method to satisfy all three desiderata.
  We address the first via influence functions~\cite{koh2017understanding}, which quantify each sample's counterfactual effect on the downstream loss and provide a theoretically grounded, task-aligned weighting scheme (\S\ref{sec:03-theory}).
  However, scoring at a single checkpoint introduces the \textit{hard-sample bias} (HSB): at convergence, noisy samples dominate influence scores while clean and moderate samples, the most valuable for a robust distilled dataset, are suppressed.
  TAKE corrects this by integrating influence scores along the full training trajectory into a per-sample knowledge score that captures clean and moderate samples while down-weighting noisy ones.
  Then, TAKE fine-tunes an LLM to generate a pool of human-readable candidate instances.
  The knowledge scores and synthetic candidates jointly feed a discrete Optimal Transport (OT) objective that globally selects the distilled corpus.
  Concretely, we contribute:
  \begin{itemize}
    \item {
        \textbf{Theory:} We formalize the hard-sample bias~(HSB) in the text DD setting, prove that single-checkpoint influence scores are biased toward noisy samples, and derive a formal gap bound for knowledge-reweighted distribution matching.
      }
    \item {
        \textbf{Method:} TAKE is the first text DD method to jointly address sample weighting and distributional coverage---via trajectory-integrated knowledge scores and a discrete OT objective aligned with the DD goal.
      }
    \item {
        \textbf{Empirical:} TAKE matches or exceeds prior text DD methods across six language benchmarks at extreme compression (0.1\% or 20 instances/class), producing human-readable distilled corpora that generalize across diverse backbone families.
      }
  \end{itemize}
\end{foldable}

The remainder of this paper is organised as follows: \S~\ref{sec:02-related} reviews related work; \S~\ref{sec:03-theory} presents theoretical motivation; \S~\ref{sec:04-method} details TAKE and our DD pipeline; \S~\ref{sec:05-experiments} reports experiments and discussion; \S~\ref{sec:06-limits} reflects limitations and ethical concerns; \S~\ref{sec:07-conclusion} concludes.

\section{Related Works} \label{sec:02-related}
\subsection{Text Dataset Distillation} \label{sec:02-related-textdd}
\begin{foldable} 
  Dataset distillation was introduced by~\cite{wang2018dataset} as a bi-level meta-learning problem: synthesize a small set of images such that a model trained on them matches one trained on the full dataset.
  Subsequent methods improve scalability and fidelity via gradient matching~\cite{zhao2020dataset}, distribution matching~\cite{wang2022cafe}, and trajectory matching~\cite{cazenavette2022dataset,kim2022dataset}.
  These methods share a common assumption: data is continuous and differentiable, so the synthetic set can be directly optimised via backpropagation.
  Text violates the differentiability assumption: token decoding is non-differentiable, outputs must be semantically coherent, and embedding spaces are tightly coupled to specific architectures.
\end{foldable}

\begin{foldable} 
  Adapting DD to text has therefore required fundamentally different strategies, from meta-learning on synthetic embeddings to LLM-based generation of readable corpora, each generation trading one limitation for another.
  SLDD and DDTC~\cite{sucholutsky2021soft,li2021data} directly adapted the meta-learning formulation~\cite{wang2018dataset}, optimizing synthetic text embeddings so that a model trained on them minimizes task loss.
  Of the two, only SLDD partially recovers interpretability by mapping distilled embeddings to per-token nearest bag-of-words.
  DDAL~\cite{maekawa2023dataset} optimizes the attention labels of a BERT model to match class-wise dataset statistics, achieving partial task alignment but producing outputs that are both architecture-locked and uninterpretable.
  A second generation of methods addressed interpretability by rethinking the pipeline entirely.
  DiLM~\cite{maekawa2024dilm} generates readable candidate texts via an LLM and selects among them by gradient matching over the training corpus, followed by k-centre selection; but matching gradients at discrete checkpoints is a proxy for task loss, leaving task alignment only partial.
  DaLLME~\cite{tao2024textual} learns an inverse mapping from distilled embeddings back to synthetic text, recovering interpretability and attempting global coverage via k-centroids clustering; but with no task-loss objective, task alignment is absent.
  Table~\ref{tab:landscape-textdd-methods} maps this progression: each generation addressed some gaps while leaving others open.
\end{foldable}

\begin{table}[ht] 
  \centering
  \caption{Landscape of text DD methods.}
  \label{tab:landscape-textdd-methods}
  \setlength{\tabcolsep}{8pt}
  \begin{tabular}{l c c c c c c}
    \toprule
    \textbf{Gap}         & SLDD         & DDTC         & DDAL         & DiLM         & DaLLME       & \textbf{TAKE} \\
    \midrule
    Task alignment       & $\checkmark$ & $\checkmark$ & $\triangle$  & $\triangle$  & $-$          & $\checkmark$ \\
    Global optimization  & $\triangle$  & $\triangle$  & $\triangle$  & $\triangle$  & $\triangle$  & $\checkmark$ \\
    Interpretability     & $\triangle$  & $-$          & $-$          & $\checkmark$ & $\checkmark$ & $\checkmark$ \\
    Importance weighting & $-$          & $-$          & $-$          & $-$          & $-$          & $\checkmark$ \\
    \bottomrule
    \multicolumn{7}{l}{$-$: absent, $\triangle$: partial, $\checkmark$: addressed}
  \end{tabular}
\end{table}

\begin{foldable} 
  A subtler gap cuts across all methods: every existing approach applies \emph{uniform weighting}, treating every training sample as equally informative.
  At moderate distillation budgets this is defensible, but extreme compression is precisely the regime where distillation matters most, and there the assumption fails.
  Samples vary substantially in their influence on learning (clean vs.\ noisy, easy vs.\ hard), yet no prior method accounts for this structure.
  Without importance weighting, the distillation budget is exhausted on uninformative samples, biasing selection away from the boundary cases that matter most for robustness.
\end{foldable}

\subsection{Influence Functions and the Hard-Sample Bias} \label{sec:02-related-influence}
\begin{foldable} 
  Influence functions (IF)~\cite{cook1982residuals,koh2017understanding} quantify the leave-one-out effect of a training point via $\mathcal{I}(z) = \nabla^\top H_\theta^{-1} \nabla$, but depend on a single checkpoint and require expensive second-order computation.
  Trajectory-based methods~\cite{pruthi2020estimating,park2023trak,kwon2023datainf} extend attribution across checkpoints but are \emph{test-conditioned}: designed to attribute a fixed test prediction.
  The self-influence approximation adapts them to corpus-level scoring by using each training point as its own query.
  However, self-influence is evaluated only once at the final parameters, introducing a systematic bias that corrupts the corpus-level scores.
\end{foldable}

\begin{foldable} 
  We refer to this as the \emph{hard-sample bias} (HSB): at convergence, clean samples have near-zero gradient by definition, so noisy or hard samples dominate influence scores.
  This failure mode is well-documented in adjacent work: \cite{arpit2017closer} show that networks fit clean patterns before noisy ones; dataset cartography~\cite{swayamdipta2020dataset} confirms that single-checkpoint statistics conflate difficulty with noise; and Co-teaching~\cite{han2018co} documents systematic overselection of noisy samples in gradient-based curricula.
  In standard training, HSB can often be mitigated via regularization; in the extreme low-data regime of DD, however, it amplifies: biased selection suppresses the \emph{informative} samples essential for model robustness.
  To our knowledge, no prior text DD method has identified or corrected for it.
\end{foldable}

\subsection{Optimal Transport for Distribution Matching} \label{sec:02-related-ot}
\begin{foldable} 
  Distribution matching has been proposed as a tractable DD surrogate~\cite{zhao2023dataset}: if the distilled and training sets induce equal expected gradients along the optimization trajectory, the trained models converge to comparable parameters.
  This is a sufficient condition, but not an equivalence (the formal gap bound appears in Theorem~\ref{thm:gap}).
  Common divergences have practical shortcomings at the extreme budgets relevant here.
  MMD~\cite{gretton2012kernel} lacks sample-level assignment structure, requiring auxiliary scoring heuristics that reintroduce the global optimization gap that distribution matching is meant to close.
  Kullback--Leibler and Jensen--Shannon divergences require density estimation, which is unreliable in the low-data regime~\cite{paninski2003estimation}.
\end{foldable}

\begin{foldable} 
  OT in the Monge-Kantorovich sense avoids both pitfalls: it operates natively on discrete measures, respects the metric geometry of the embedding space, and produces an explicit transport plan that directly assigns training samples to prototypes---no auxiliary heuristic needed.
  Entropic regularization via Sinkhorn-Knopp~\cite{cuturi2013sinkhorn} makes this tractable at corpus scale.
  OT has been widely applied in NLP and machine learning~\cite{kusner2015word,tolstikhin2018wasserstein,courty2016optimal,arjovsky2017wasserstein}, but never to text DD, and never combined with a knowledge-reweighted source distribution in any DD setting.
  TAKE fills both gaps: $\kappa_n$ defines a non-uniform source measure, and the discrete entropic OT plan selects prototypes covering the knowledge-weighted training distribution without any auxiliary selection step.
\end{foldable}

\section{Theoretical Motivation} \label{sec:03-theory}
\subsection{Reweighted Distribution Matching for Dataset Distillation} \label{sec:03-theory-rdm}
\begin{foldable} 
  The dataset distillation~(DD) objective seeks $\tilde{\mathcal{D}}$ with $|\tilde{\mathcal{D}}| = K \ll N$ minimizing the loss on the original data:
  \begin{equation}
    \min_{\tilde{\mathcal{D}}}\; \mathbb{E}_{\theta_0 \sim p(\theta_0)} \left[ \mathbb{E}_{z \sim \mathcal{D}} \left[ \ell\!\left(z;\; F(\theta_0;\, \tilde{\mathcal{D}}, \eta, T)\right) \right] \right]
    \label{eq:dataset-distillation}
  \end{equation}
  where $F(\theta_0; \tilde{\mathcal{D}}, \eta, T)$ denotes the model obtained by running $T$ steps of gradient descent with learning rate $\eta$ on $\tilde{\mathcal{D}}$ from initialization $\theta_0$.
  This bi-level problem is intractable for discrete text.
  We relax it to distribution matching via the gradient-field equivalence argument~\cite{zhao2023dataset}: if $\tilde{\mathcal{D}}$ induces the same expected gradient field as $\mathcal{D}$ throughout training, the two runs converge to comparable parameters.
  Since $\nabla_\theta \mathbb{E}_{z \sim P}[\ell(z;\theta)] = \mathbb{E}_{z \sim P}[\nabla_\theta \ell(z;\theta)]$, matching distributions in a task-relevant feature space controls the gradient field, making distribution matching a tractable surrogate for~Eq.~\ref{eq:dataset-distillation}.
  This yields the distribution matching~(DM) objective:
  \begin{equation}
    \min_{|\tilde{\mathcal{D}}| = K}\; d\!\left(P_{\tilde{\mathcal{D}}},\; P_{\mathcal{D}}\right)
    \label{eq:distribution-matching}
  \end{equation}
  for a divergence $d$ in a task-relevant feature space.
  Eq.~\ref{eq:distribution-matching} is a surrogate, not an equivalence; the formal gap bound appears as Theorem~\ref{thm:gap} once $d$ is specified.
\end{foldable}

\begin{foldable} 
  At extreme compression ($\rho = K/N \ll 1$), uniform matching allocates prototypes proportional to density --- but not all regions contribute equally to learning.
  Dense easy regions consume budget that would be better spent on moderate and boundary samples near decision boundaries, which carry the most information for model robustness.
  We instead reweight the source by per-sample importance: define $w_n \geq 0$, $\sum_n w_n = 1$, inducing $P^w = \sum_n w_n \delta_{z_n}$, yielding the reweighted distribution matching~(RDM) objective:
  \begin{equation}
    \min_{|\tilde{\mathcal{D}}| = K}\; d\!\left(P_{\tilde{\mathcal{D}}},\; P^w\right)
    \label{eq:reweighted-distribution-matching}
  \end{equation}
  The weights $w_n$ must reflect downstream learning contribution, not geometry alone, motivating gradient-based estimation.
  Influence self-scores (\S\ref{sec:02-related-influence}) are the natural candidate: they are task-aligned by construction and theoretically grounded.
  We now show that the simplest instantiation---$w_n \propto \mathcal{I}_\text{self}(z_n)$ at a single checkpoint, select top-$K$---fails on both counts: the weights are biased toward noisy samples (Proposition~\ref{prop:hsb}), and independent top-$K$ selection ignores distributional coverage (Remark~\ref{rem:coverage}).
  \begin{remark}[Double relaxation gap] \label{rem:double-gap}
    Moving from Eq.~\ref{eq:dataset-distillation} to Eq.~\ref{eq:distribution-matching} incurs a gradient-field approximation error (bounded formally by Theorem~\ref{thm:gap}); moving from~Eq.~\ref{eq:distribution-matching} to Eq.~\ref{eq:reweighted-distribution-matching} replaces the uniform source $P_\mathcal{D}$ with the reweighted source $P^w$, incurring an additional cost:
    \begin{equation}
      \mathrm{Gap}_2 \;\leq\; L \cdot W_1(P_\mathcal{D},\, P^w) \;=\; L \cdot W_1\!\left(\tfrac{1}{N}\textstyle\sum_n \delta_{z_n},\; \textstyle\sum_n w_n \delta_{z_n}\right).
      \label{eq:gap2}
    \end{equation}
    Since both measures have the same support, $W_1(P_\mathcal{D}, P^w) \leq \mathrm{diam}(\mathcal{Z}) \cdot \frac{1}{2}\sum_n |w_n - 1/N|$, which is directly computable from the $\kappa_n$ scores and small when $w_n \approx 1/N$ (uniform weights recover DM).
    The reweighting is beneficial precisely when this gap is outweighed by the reduction in downstream loss from correcting the HSB, as formalized in Proposition~\ref{prop:traj} (\S\ref{sec:04-method-gradient}).
  \end{remark}
\end{foldable}

\subsection{The Naive Instantiation Fails: Two Open Problems} \label{sec:03-theory-failures}
\begin{foldable} 
  Define the influence self-score at checkpoint $\hat{\theta}$ as
  \begin{equation}
    \mathcal{I}_\text{self}(z_n; \hat{\theta}) = \nabla_\theta \ell(z_n; \hat{\theta})^\top H_{\hat{\theta}}^{-1} \nabla_\theta \ell(z_n; \hat{\theta}),
    \label{eq:influence-self}
  \end{equation}
  where $H_{\hat{\theta}} = \frac{1}{N}\sum_n \nabla^2_\theta \ell(z_n; \hat{\theta})$ is the empirical Hessian, specializing~\cite{koh2017understanding} by setting the test point equal to the training point.
  High $\mathcal{I}_\text{self}$ identifies atypical examples: samples that are difficult to fit and exert outsized influence on the model's parameters, which in practice are the noisy ones.
  This suggests a naive pipeline: set $w_n \propto \mathcal{I}_\text{self}(z_n; \hat{\theta})$ and return $\arg\max_{|\mathcal{A}|=K} \sum_{i \in \mathcal{A}} w_i$.
  We identify two orthogonal failures of this strategy.
\end{foldable}

\begin{foldable} 
  The first failure is a bias in the weights themselves.
  \begin{proposition}[Hard-Sample Bias] \label{prop:hsb}
    Let $\mathcal{D}_\text{clean},\, \mathcal{D}_\text{noisy}$ partition $\mathcal{D}$, under the assumptions:
    \emph{(i)}~$\|\nabla_\theta \ell(z; \hat{\theta})\| \leq \epsilon$ for $z \in \mathcal{D}_\text{clean}$, i.e., clean samples have near-zero gradient at convergence;
    \emph{(ii)} $\|\nabla_\theta \ell(z; \hat{\theta})\| \geq \delta \gg \epsilon$ for $z \in \mathcal{D}_\text{noisy}$, i.e., noisy samples retain large residual gradients.
    Let $\mathrm{cond}(H) = \lambda_{\max}(H) / \lambda_{\min}(H)$ denotes the condition number of the Hessian $H$.
    Then:
    \begin{equation}
      \Delta_\textup{HSB} := \mathbb{E}_{z \sim \mathcal{D}_\textup{noisy}}[\mathcal{I}_\textup{self}(z;\hat{\theta})]
      - \mathbb{E}_{z \sim \mathcal{D}_\textup{clean}}[\mathcal{I}_\textup{self}(z;\hat{\theta})]
      \;\geq\; \frac{\delta^2 - \mathrm{cond}(H_{\hat{\theta}})\,\epsilon^2}{\lambda_{\max}(H_{\hat{\theta}})}
      \;>\; 0
      \label{eq:hsb}
    \end{equation}
  \end{proposition}
  \begin{proof}[Proof sketch (full proof: Appendix~A.1)]
    Since $\mathcal{I}_\text{self}(z) = \nabla^\top H_{\hat\theta}^{-1} \nabla$, spectral bounds on $H_{\hat\theta}^{-1}$ give
    $\|\nabla\|^2/\lambda_{\max}(H_{\hat\theta}) \leq \mathcal{I}_\text{self}(z) \leq \|\nabla\|^2/\lambda_{\min}(H_{\hat\theta})$.
    Applying the lower bound to noisy samples and the upper bound to clean samples yields Eq.~\ref{eq:hsb}.
    Positivity requires $\delta/\epsilon > \sqrt{\mathrm{cond}(H_{\hat\theta})}$, which is mild given $\delta \gg \epsilon$.
  \end{proof}
  Because clean gradients vanish at convergence by definition of a local minimum, the bias is \emph{structural}: it cannot be corrected by post-hoc rescaling of single-checkpoint scores, but only by integrating gradient signals across the training trajectory.
  The binary partition simplifies exposition; in practice the bias is monotone in $\|\nabla\|$, as the spectral bounds above show directly.
  \begin{remark}[Gradient norm as special case] \label{rem:grad-norm}
    Setting $H = I$ in Eq.~\ref{eq:influence-self} recovers $\mathcal{I}_\text{self}(z) = \|\nabla\|^2$.
    Under this substitution, Eq.~\ref{eq:hsb} reduces to $\Delta_\text{HSB} \geq \delta^2 - \epsilon^2 > 0$, which holds without any spectral assumption on $H$.
    The gradient norm used in TAKE (\S\ref{sec:04-method-gradient}) therefore inherits the same structural bias: it is not merely a tractable proxy but satisfies the HSB lower bound directly.
  \end{remark}
\end{foldable}

\begin{foldable} 
  Even with unbiased weights, a second failure remains: top-$K$ selection is blind to distributional coverage.
  \begin{remark}[Independent scoring $\neq$ distributional coverage] \label{rem:coverage}
    Any rank-and-select strategy decomposes as a sum of independent terms, invariant to pairwise distances between selected samples.
    In contrast, $d(P_{\tilde{\mathcal{D}}}, P^w)$ penalises redundancy and rewards coverage.
    Concretely, if $\phi(z_i) \approx \phi(z_j)$ and both have high $w_i, w_j$, independent scoring selects both; a distributional objective assigns nearly the same cost to selecting just one, freeing a slot for an uncovered region.
  \end{remark}
\end{foldable}

\begin{foldable} 
  The two failures compound: biased $w_n$ misdirects a population-level solver toward noisy regions; correct $w_n$ fed into an independent selector still produces redundancy.
  They respectively yield two open problems:
  \begin{enumerate}[label=\textbf{(\alph*)}]
    \item \label{open-problem-weight} \textbf{Biased weighting.} Estimate $w_n$ by integrating gradient signals across trajectory $\{\theta_t\}_{t=0}^{T-1}$, reducing $\Delta_\textup{HSB}$---which we will address in \S\ref{sec:04-method-gradient} (formalized as Proposition~\ref{prop:traj}; Appendix~A.2).
    \item \label{open-problem-select} \textbf{Independent selection.} Solve Eq.~\ref{eq:reweighted-distribution-matching} directly for a metric $d$ that jointly optimises coverage of $P^w$---which we will address in \S\ref{sec:04-method-ot}.
  \end{enumerate}
\end{foldable}

\section{Methodology} \label{sec:04-method}
\subsection{Problem Setup and Framework Overview} \label{sec:04-method-setup}
\begin{foldable} 
  TAKE instantiates the RDM objective (Eq.~\ref{eq:reweighted-distribution-matching}) by solving the two open problems from \S\ref{sec:03-theory-failures}.
  Let $\mathcal{D} = \{z_n\}_{n=1}^{N}$ denote the training corpus, $\mathcal{D}_\text{synth} = \{\tilde{z}_c\}_{c=1}^{M}$ the synthetic candidate pool, and $\tilde{\mathcal{D}} \subset \mathcal{D}_\text{synth}$ the distilled set of budget $K$.
  Indices $n \in [N]$ and $t \in [T]$ range over training samples and checkpoints respectively.
  TAKE proceeds in three stages:
  \begin{enumerate}
    \item \textbf{Score}---train a logistic probe on $\mathcal{D}$ across $T$ checkpoints, compute influence matrix $I \in \mathbb{R}^{N \times T}$, apply reciprocal reweighting and temporal weighting to obtain trajectory-aware knowledge scores $\kappa_n$.
    \item \textbf{Generate}---fine-tune a small language model on $\mathcal{D}$ and sample to produce the synthetic candidate pool $\mathcal{D}_\text{synth}$.
    \item \textbf{Distill}---solve discrete OT with source weights $\kappa_n$ over $\mathcal{D}_\text{synth}$ to obtain the distilled corpus $\tilde{\mathcal{D}}$ of budget $K$.
  \end{enumerate}
  The three stages are decoupled and independently replaceable, with $\phi$ shared across Stages 1 and 3 as a natural consistency choice. The following sections detail each in turn.
\end{foldable}

\subsection{Gradient-Based Influence Along the Trajectory} \label{sec:04-method-gradient}
\begin{foldable} 
  The hard-sample bias (HSB; \S\ref{sec:03-theory-failures}) arises because single-checkpoint influence scores overemphasize noisy or difficult samples, whose gradients remain large at convergence.
  The fix is to integrate gradient information across the full training trajectory, producing a trajectory-aware knowledge score $\kappa_n$ per sample.
  We build $\kappa_n$ in two steps: constructing the influence matrix $I$ (\S\ref{sec:04-gradient-matrix}), then collapsing it into a per-sample scalar score via reciprocal reweighting and temporal weighting (\S\ref{sec:04-gradient-kappa}).
  The resulting $\kappa_n$ provides a task-aligned, bias-corrected weight for $P^w$.
\end{foldable}

\subsubsection{Influence Matrix} \label{sec:04-gradient-matrix}
\begin{foldable} 
  Our goal is to construct $I \in \mathbb{R}^{N \times T}$ whose entry $I_{(n,t)}$ reflects how much sample $z_n$ contributes to learning at checkpoint $\theta_t$.
  The influence self-score (\S\ref{sec:03-theory-failures}, Eq.~\ref{eq:influence-self}) is the natural candidate, but requires Hessian inversion at a converged checkpoint and is undefined mid-trajectory.
  We address both limitations by adapting the per-sample gradient norm from TracIn~\cite{pruthi2020estimating} to the self-influence setting, querying training points instead of test points:
  \begin{equation}
    I_{(n,t)} = \|\nabla_\theta \ell(z_n;\,\theta_t)\|.
    \label{eq:influence-matrix}
  \end{equation}
  As the $H=I$ special case of the influence self-score (Remark~\ref{rem:grad-norm}), $I_{(n,t)}$ exhibits the same HSB structure---the bias that $\kappa_n$ corrects via reciprocal inversion and trajectory integration (Proposition~\ref{prop:traj}).
  Scores are column-wise normalized $I_{(n,t)} \leftarrow I_{(n,t)} / \sum_{n'} I_{(n',t)}$ for comparability across checkpoints.
\end{foldable}

\begin{foldable}
  In practice, we score samples with a linear head $W$ atop a pre-trained, frozen sentence encoder $\phi$, giving $\theta = (W, \phi)$.
  The signal quality comes not from the linear head but from $\phi(x_n)$: $\phi$ is pre-trained on large-scale corpora and $W$ is fine-tuned on the full task corpus, so the representations, and hence the gradient norms, are both semantically rich and task-aligned by construction.
  Following the standard practice in scalable influence estimation~\cite{guo2021fastif}, we approximate full-model influence via the last layer---here, the linear head $W$:
  \begin{equation}
    \nabla_\theta \ell(z_n;\,\theta_t) \approx \nabla_W \ell(z_n;\,\theta_t),
    \label{eq:last-layer}
  \end{equation}
  trading a small bias for tractability.
  By the Goodfellow outer-product identity~\cite{goodfellow2015efficient}, $\nabla_W \ell = g_n \phi(x_n)^\top$ where $g_n$ is the output gradient.
  The Goodfellow trick vectorizes the batch backward pass, avoiding per-sample loops.
  Adapting to a new task is trivial: retrain $W$ on the new corpus; $\phi$ and the influence pipeline are unchanged.
\end{foldable}

\subsubsection{Knowledge Score and Temporal Weighting} \label{sec:04-gradient-kappa}
\begin{foldable} 
  To correct the HSB, we require a score that is high for well-fitted samples and low for noisy ones---the opposite of what $I_{(n,t)}$ directly provides.
  Since a well-fitted sample has a small gradient norm, taking the reciprocal naturally inverts this ranking.
  We further weight each checkpoint by a decreasing temporal kernel $k : \{0,\ldots,T-1\} \to \mathbb{R}_{>0}$, giving:
  \begin{equation}
    \kappa_n = \sum_{t=0}^{T-1} \frac{k(t)}{I_{(n,t)}},
    \label{eq:kappa}
  \end{equation}
\end{foldable}

\begin{foldable}
  A decreasing kernel is the natural choice: early checkpoints capture easy samples fitting rapidly, mid checkpoints offer the strongest signal, and late checkpoints see the gap narrow as memorization sets in~\cite{arpit2017closer}.
  Thus, concentrating mass on the early-to-mid phase and discounting late memorization directly corrects the HSB structure identified in Proposition~\ref{prop:hsb}.
  By default we use the exponential kernel $k(t) = e^{-\lambda t / T}$ ($\lambda > 0$), with $\lambda$ treated as a hyperparameter.
  Alternative kernels (linear, cosine) are drop-in replacements, and Proposition~\ref{prop:traj} guarantees bias reduction holds for any strictly decreasing $k$.
  The resulting $\kappa_n$ summarizes each sample's cumulative, bias-corrected learning contribution across the full trajectory, resolving open problem~\ref{open-problem-weight}.
  Normalized source weights $w_n = \kappa_n / \sum_{n'} \kappa_{n'}$ define the reweighted distribution $P^w$ and serve as input to the OT selection step (\S\ref{sec:04-method-ot}).
\end{foldable}

\begin{foldable}
  \begin{proposition}[Trajectory Integration Reduces HSB] \label{prop:traj}
    Under a smooth loss and bounded gradient noise, $\kappa_n$ produces a strictly smaller noisy-over-clean gap than the single last-checkpoint baseline $\kappa^\dagger_n = 1/I_{(n,T-1)}$:
    \begin{equation}
      \mathbb{E}_{z \sim \mathcal{D}_\textup{noisy}}[\kappa_n] - \mathbb{E}_{z \sim \mathcal{D}_\textup{clean}}[\kappa_n]
      < \mathbb{E}_{z \sim \mathcal{D}_\textup{noisy}}[\kappa^\dagger_n] - \mathbb{E}_{z \sim \mathcal{D}_\textup{clean}}[\kappa^\dagger_n].
      \label{eq:prop2}
    \end{equation}
  \end{proposition}
  \begin{proof}[Proof sketch (full proof: Appendix~A.2)]
    Define $\Delta_t := \mathbb{E}_\textup{noisy}[1/I_{(n,t)}] - \mathbb{E}_\textup{clean}[1/I_{(n,t)}]$ as the noisy-over-clean gap.
    The left-hand side of Eq.~\ref{eq:prop2} equals $\sum_t k(t)\Delta_t$.
    For $t \leq T_m$ (before memorization onset): clean gradients decay so $\Delta_t \leq 0$, strictly at some $t^*$ (Appendix~A.2).
    Beyond $T_m$: noisy gradients also drop (memorization), so $\Delta_{T-1} > 0$ — the last-checkpoint baseline is inflated by memorization.
    The decreasing kernel upweights $t \leq T_m$ (where $\Delta_t < 0$) and downweights $t > T_m$ (where $\Delta_t > 0$).
    Since $k$ is strictly decreasing, the negative-$\Delta_t$ early terms receive more weight and the large positive $\Delta_{T-1}$ term receives the least, making the weighted sum strictly smaller than $\Delta_{T-1} > 0$.
  \end{proof}
\end{foldable}

\subsection{Synthetic Candidate Pool Generation} \label{sec:04-method-pool}
\begin{foldable} 
  Rather than selecting directly from $\mathcal{D}$, TAKE draws $\tilde{\mathcal{D}}$ from a synthetic candidate pool $\mathcal{D}_\text{synth}$.
  Direct subsampling has two problems: at extreme compression, verbatim records raise privacy concerns, and the budget $K$ is too small to cover the full training support.
  We generate $\mathcal{D}_\text{synth}$ by fine-tuning a small language model on $\mathcal{D}$ and sampling from it.
  The fine-tuned LM generalizes beyond memorized training instances, and stochastic sampling produces a diverse pool that extends beyond the training support.
  $\mathcal{D}_\text{synth}$ serves solely as the selection pool and is never appended to downstream training data.
\end{foldable}

\begin{foldable} 
  TAKE does \emph{not} assume $P_{\mathcal{D}_\text{synth}} = P_{\mathcal{D}}$; this distributional gap is \emph{beneficial}.
  First, coverage expansion fills undersampled regions of $\mathcal{D}$, improving generalization of the distilled corpus.
  Second, model fluency bias suppresses noisy samples that HSB would otherwise over-select, yielding a cleaner candidate pool.
  Third, the distribution mismatch physically decouples $\tilde{\mathcal{D}}$ from original records, reducing verbatim exposure risk.
  OT handles the distributional gap gracefully: it routes mass toward $P^w$ regardless of pool distribution, so $P_{\tilde{\mathcal{D}}}$ approximates $P^w$ provided $\mathcal{D}_\text{synth}$ covers the support of $P_{\mathcal{D}}$ (Assumption~\ref{ass:coverage}):
  \begin{assumption}[Pool Coverage] \label{ass:coverage}
    For every training sample $z_n \in \mathcal{D}$, there exists a synthetic sample $\tilde{z}_c \in \mathcal{D}_\text{synth}$ within embedding distance $r$:
    $\min_c \|\phi(z_n) - \phi(\tilde{z}_c)\| \leq r$.
  \end{assumption}
  Under Assumption~\ref{ass:coverage}, $W_1(P^w, P_{\mathcal{D}_\text{synth}}) \leq r$, so the optimal transport cost is at most $r$ before any prototype selection.
  A practitioner can monitor coverage directly by reporting the mean nearest-neighbor distance $\bar{r} = \frac{1}{N}\sum_n \min_c \|\phi(z_n) - \phi(\tilde{z}_c)\|$.
\end{foldable}

\subsection{Distributional Prototype Selection via Discrete Optimal Transport} \label{sec:04-method-ot}

\begin{foldable} 
  Both $\mathcal{D}$ and $\mathcal{D}_\text{synth}$ are finite sample sets, so the matching problem is naturally over empirical measures.
  Discrete OT operates directly on $\mu = \sum_n w_n \delta_{z_n}$ and $\nu = \frac{1}{|\mathcal{D}_\text{synth}|}\sum_c \delta_{\tilde{z}_c}$, requiring no density estimation and yielding an explicit transport plan with direct prototype-level assignments.
\end{foldable}

\begin{foldable} 
  Let $C_{nc} = \|\phi(z_n) - \phi(\tilde{z}_c)\|^2$ be the cost matrix in a task-relevant embedding space.
  Bare Kantorovich OT, $\min_{\gamma \in \Pi(\mu,\nu)} \langle C, \gamma \rangle$, produces a sparse, non-differentiable plan and scales as $O(N^3)$.
  We add entropic regularization to obtain a strictly convex, smooth objective solvable via Sinkhorn--Knopp iterations at $O(N \cdot |\mathcal{D}_\text{synth}|)$ per step:
  \begin{equation}
    \gamma^* = \arg\min_{\gamma \in \Pi(\mu,\nu)} \langle C, \gamma \rangle + \varepsilon \cdot \mathrm{KL}(\gamma \| \mu \otimes \nu).
    \label{eq:ot}
  \end{equation}
  The soft coupling distributes mass across nearby candidates rather than hard one-to-one assignment, reducing sensitivity to embedding errors and encouraging diversity in $\tilde{\mathcal{D}}$.
  The $K$ candidates with highest received mass $\sum_n \gamma^*_{nc}$ are selected as $\tilde{\mathcal{D}}$, resolving open problem~\ref{open-problem-select}.
  The following theorem closes the theoretical chain from RDM (Eq.~\ref{eq:reweighted-distribution-matching}) back to DD (Eq.~\ref{eq:dataset-distillation}), bounding the downstream loss gap in terms of the OT cost.
  \begin{theorem}[Distribution Matching Gap Bound] \label{thm:gap}
    Let $f^*_{\mathcal{S}}$ denote the model trained on $\mathcal{S}$, and let $\mathcal{L}(\mathcal{S}) = \mathbb{E}_{z \sim P^w}[\ell(z;\,f^*_{\mathcal{S}})]$.
    Assume: \emph{(i)} $\ell(\cdot;\,f)$ is $L$-Lipschitz in the embedding metric for all $f$; \emph{(ii)} the model class has uniform stability constant $\beta_K = O(1/K)$~\cite{bousquet2002stability}, so that $\sup_z |\ell(z;\,f^*_{\tilde{\mathcal{D}}}) - \ell(z;\,f^*_\mathcal{D})| \leq \beta_K$.
    Then:
    \begin{equation}
      \bigl|\mathcal{L}(\tilde{\mathcal{D}}) - \mathcal{L}(\mathcal{D})\bigr|
      \;\leq\; 2L \cdot W_1(P^w,\,P_{\tilde{\mathcal{D}}}) + \beta_K,
      \label{eq:gap-bound}
    \end{equation}
    where $W_1$ is the Wasserstein-1 distance under the embedding metric.
    As $K \to \infty$, $\beta_K \to 0$, and the bound is dominated by the $W_1$ term controlled by the OT objective.
  \end{theorem}
  \begin{proof}[Proof sketch (full proof: Appendix~A.3)]
    Insert two intermediate terms, $\mathbb{E}_{P_{\tilde{\mathcal{D}}}}[\ell(z;\,f^*_{\tilde{\mathcal{D}}})]$ and $\mathbb{E}_{P_{\tilde{\mathcal{D}}}}[\ell(z;\,f^*_\mathcal{D})]$, and apply the triangle inequality twice.
    This yields two distribution-shift terms (under fixed $f^*_{\tilde{\mathcal{D}}}$ and $f^*_\mathcal{D}$ respectively), each bounded by $L \cdot W_1(P^w, P_{\tilde{\mathcal{D}}})$ via Kantorovich--Rubinstein duality; and one model-difference term under $P_{\tilde{\mathcal{D}}}$, bounded by $\beta_K = O(1/K)$ via uniform stability.
  \end{proof}
  The bound is stated with respect to $P^w$ rather than the uniform $P_\mathcal{D}$ --- this is intentional, as $P^w$ is the designed optimization target of the RDM objective; since $P^w$ reweights $P_\mathcal{D}$ over the same support, minimizing $W_1(P^w, P_{\tilde{\mathcal{D}}})$ controls coverage of $P_\mathcal{D}$ up to the reweighting gap bounded in Remark~\ref{rem:double-gap}.
\end{foldable}

\begin{foldable} 
  Together, Propositions~\ref{prop:hsb}--\ref{prop:traj} and Theorem~\ref{thm:gap} establish the full theoretical basis for TAKE: HSB is structural and unavoidable at a single checkpoint (Proposition~\ref{prop:hsb}); trajectory integration strictly reduces it (Proposition~\ref{prop:traj}); and minimizing the OT cost directly controls the downstream loss gap under $P^w$ (Theorem~\ref{thm:gap}; the gap to $P_\mathcal{D}$ is bounded by Remark~\ref{rem:double-gap}).
  Algorithm~\ref{alg:take} consolidates the full TAKE pipeline.
\end{foldable}

\begin{center}
  \begin{minipage}{0.99\linewidth}
    \begin{algorithm}[H]
      \SetKwInOut{KwIn}{Input}
      \SetKwInOut{KwOut}{Output}
      \SetKwComment{Comment}{}{}

      \caption{Trajectory-Aware Knowledge Estimation (\textbf{TAKE})}
      \label{alg:take}

      \KwIn{corpus $\mathcal{D} = \{z_n\}_{n=1}^{N}$; budget $K$; checkpoints $T$; kernel $k(\cdot)$; $\varepsilon$; pool size $M$}
      \KwOut{distilled corpus $\tilde{\mathcal{D}}$, $|\tilde{\mathcal{D}}| = K$}

      \vspace{-0.5\baselineskip}
      \Comment{\noindent\hrulefill}

      \vspace{0.5\baselineskip}
      \Comment{$\triangleright$ \textbf{\textit{Stage 1: Score}}}
      Train logistic probe on $\mathcal{D}$, saving checkpoints $\theta_0, \ldots, \theta_{T-1}$ \\
      \For{$t = 0, \ldots, T-1$}{
        Compute $I_{(n,t)} \leftarrow \|\nabla_W \ell(z_n;\,\theta_t)\|$ for all $n$ \hfill\textit{(Goodfellow trick)} \\
        Normalize $I_{(\cdot,t)} \leftarrow I_{(\cdot,t)} / \sum_{n'} I_{(n',t)}$ \hfill\textit{(column-wise)}
      }
      Compute $\kappa_n \leftarrow \textstyle\sum_{t=0}^{T-1} k(t) / I_{(n,t)}$ for all $n$ \hfill via Eq.~\ref{eq:kappa} \\
      Normalize $w_n \leftarrow \kappa_n / \sum_{n'} \kappa_{n'}$ for all $n$ \hfill\textit{(defines $P^w$)}

      \vspace{0.5\baselineskip}
      \Comment{$\triangleright$ \textbf{\textit{Stage 2: Generate}}}
      Fine-tune small language model $\mathrm{LM}$ on $\mathcal{D}$ \\
      Sample synthetic pool $\mathcal{D}_\text{synth} = \{\tilde{z}_c\}_{c=1}^{M}$ from $\mathrm{LM}$

      \vspace{0.5\baselineskip}
      \Comment{$\triangleright$ \textbf{\textit{Stage 3: Distill}}}
      Embed: $\phi(z_n)$ for $n \in [N]$; $\phi(\tilde{z}_c)$ for $c \in [M]$ \\
      Compute cost matrix $C_{nc} = \|\phi(z_n) - \phi(\tilde{z}_c)\|^2$ \\
      Solve $\gamma^* \leftarrow \mathrm{Sinkhorn}(C,\, w,\, \mathbf{1}/M,\, \varepsilon)$ \hfill via Eq.~\ref{eq:ot} \\
      Distill $\tilde{\mathcal{D}} \leftarrow \mathrm{top}\text{-}K\!\left(\{\tilde{z}_c\},\; \textstyle\sum_n \gamma^*_{nc}\right)$ \hfill\textit{(by received mass)}

      \vspace{0.5\baselineskip}
      \Return{$\tilde{\mathcal{D}}$}
    \end{algorithm}
  \end{minipage}
\end{center}

\section{Experiments} \label{sec:05-experiments}
\subsection{Setup} \label{sec:05-setup}
\begin{foldable} 
  We evaluate on two linguistic tasks: AG News ($N=120$\,K), IMDb ($N=25$\,K), SST-2 ($N=67$\,K) for classification; MNLI-m ($N=393$\,K), QNLI ($N=105$\,K), and QQP ($N=364$\,K) for natural language inference (NLI).
  We compare against four baselines:
  \begin{itemize}
    \item \textbf{Random}: uniform random selection at the target budget.
    \item \textbf{EDA}~\cite{wei2019eda}: word-level augmentation (10$\times$) applied to the \textbf{Random} set.
    \item \textbf{DiLM}~\cite{maekawa2024dilm}: LLM-generated candidates selected by gradient matching across k-centre selection. It evaluated on both classification and NLI.
    \item \textbf{DaLLME}~\cite{tao2024textual}: synthetic texts inverted from distilled embeddings. It evaluated on classification only.
  \end{itemize}
  We evaluate TAKE at two budgets (20/cls and 0.1\%) to match DiLM's and DaLLME's reported settings, respectively.
  To match downstream evaluation, downstream models for classification are Logistic Regression~(LR) with TF-IDF~\cite{salton1988term} features, TextCNN~\cite{kim2014convolutional} and TextRNN~\cite{liu2016recurrent} with GloVe~\cite{pennington2014glove} word embeddings, and BERT~\cite{devlin2019bert}.
  For NLI, we evaluate on BERT and Siamese LR.
  Siamese LR is a variant of LR: two input branches to independently encode sentence pairs, and a logistic head over the concatenated representations, tailored as a weak learner for NLI.
  The primary metric is downstream accuracy on classification/NLI tasks, reported as \meanstd{mean}{\text{s.e.}} across five runs.
\end{foldable}

\begin{foldable} 
  In terms of implementation, TAKE extracts knowledge scores $\kappa_n$ by running $T=20$ training checkpoints with exponential kernel $\lambda^* = T/T_m$, using all-MiniLM-L6-v2~\cite{wang2020minilm} as the frozen encoder $\phi$.
  The synthetic pool uses Gemma-3~\cite{kamath2025gemma}~(270\,M) fine-tuned on $\mathcal{D}$ to generate $|\mathcal{D}_\text{synth}|=5{,}000$ samples.
  OT selection embeds via the same $\phi$ and uses Sinkhorn with $\varepsilon=0.05$ for 200 iterations.
  All experiments fit on a single NVIDIA V100~(40\,GB).
\end{foldable}

\subsection{Classification} \label{sec:05-results-classification}
\begin{table}[ht]
  \centering
  \begin{threeparttable}
    \caption{Classification accuracy (\%) on AG News, IMDb, and SST-2.}
    \label{tab:results-classification}
    \setlength{\tabcolsep}{4pt}
    \begin{tabular}{l l c c c c c}
      \toprule
      \textbf{Method} & \textbf{Budget} & \textbf{LR} & \textbf{TextCNN} & \textbf{TextRNN} & \textbf{BERT} \\
      \midrule
      \multirow{5}{*}{\rotatebox{90}{AG News}}
      & Full dataset  & 100\% & \meanstd{90.21}{0.14}    & \meanstd{91.27}{0.20}    & \meanstd{90.98}{0.32}      & \meanstd{93.85}{0.12} \\
      & Random        & 0.1\% & \meanstd{66.39}{2.27}    & \meanstd{75.56}{0.14}    & \meanstd{74.91}{1.45}      & \meanstd{78.12}{3.25} \\
      & EDA           & 0.1\% & \meanstd{68.03}{2.43}    & \meanstd{82.27}{0.15}    & \meanstd{81.46}{1.34}      & \meanstd{84.45}{1.44} \\
      & DaLLME        & 0.1\% & \meanstd{74.60}{0.02}\dg & \meanstd{88.30}{0.02}\dg & \meanstdbf{84.50}{0.02}\dg & $-$ \\
      & \textbf{TAKE} & 0.1\% & \meanstdbf{77.07}{0.09}  & \meanstdbf{88.53}{0.21}  & \meanstd{84.25}{1.03}      & \meanstdbf{89.82}{0.18} \\
      \midrule
      \multirow{5}{*}{\rotatebox{90}{IMDb}}
      & Full dataset  & 100\% & \meanstd{88.41}{0.13}    & \meanstd{87.40}{0.20}    & \meanstd{83.00}{0.35}      & \meanstd{91.22}{0.25} \\
      & Random        & 0.1\% & \meanstd{59.73}{5.51}    & \meanstd{57.65}{1.19}    & \meanstd{55.09}{4.31}      & \meanstd{60.40}{4.10} \\
      & EDA           & 0.1\% & \meanstd{61.52}{4.62}    & \meanstd{62.40}{3.40}    & \meanstd{58.92}{1.93}      & \meanstd{63.12}{1.95} \\
      & DaLLME        & 0.1\% & \meanstd{65.00}{0.02}\dg & \meanstd{68.30}{0.02}\dg & \meanstdbf{64.50}{0.02}\dg & $-$ \\
      & \textbf{TAKE} & 0.1\% & \meanstdbf{66.83}{0.13}  & \meanstdbf{68.93}{2.15}  & \meanstd{63.26}{0.38}      & \meanstdbf{71.55}{0.22} \\
      \midrule
      \multirow{5}{*}{\rotatebox{90}{SST-2}}
      & Full dataset  & 100\%  & \meanstd{80.23}{0.27}   & \meanstd{88.81}{0.15}    & \meanstd{87.40}{0.76}      & \meanstd{92.52}{0.28} \\
      & Random        & 20/cls & \meanstd{55.28}{2.32}   & \meanstd{63.85}{0.87}    & \meanstd{61.52}{1.66}      & \meanstd{69.97}{3.21} \\
      & EDA           & 20/cls & \meanstd{57.21}{1.05}   & \meanstd{65.54}{0.45}    & \meanstd{63.17}{0.84}      & \meanstd{74.21}{1.58} \\
      & DiLM          & 20/cls & $-$                     & $-$                      & $-$                        & \meanstd{80.30}{2.80}\dg \\
      & \textbf{TAKE} & 20/cls & \meanstdbf{65.05}{0.05} & \meanstdbf{70.65}{0.32}  & \meanstdbf{67.18}{0.49}    & \meanstdbf{81.15}{0.24} \\
      \bottomrule
    \end{tabular}
    \begin{tablenotes} \footnotesize
    \item {
        [$-$]~not reported; [$\dagger$]~reported from respective paper; [\textbf{bold}]~best per budget. \phantom{1234}
        DaLLME reported with its best case (OpenAI-3-large embeddings).
      }
    \end{tablenotes}
  \end{threeparttable}
\end{table}

\begin{foldable} 
  \textbf{Accuracy gains and their source.}
  TAKE (0.1\%) outperforms DaLLME on LR by $+2.5$ on AG News and $+1.8$ on IMDb.
  The gap is widest at the LR level and narrows for neural backbones, consistent with TAKE's backbone-agnostic selection---trajectory scores are computed independently of the downstream model.
  TAKE minimizes the OT cost in embedding space, directly controlling the downstream loss gap under $P^w$ via Theorem~\ref{thm:gap} (the gap to the original $P_\mathcal{D}$ is bounded by Remark~\ref{rem:double-gap}).
  Since $\kappa_n$ scores are computed independently of the evaluation backbone, this coverage of $P^w$ transfers across model families.
\end{foldable}

\begin{foldable}
  \textbf{Stability across runs.}
  The wide standard errors for Random (IMDb LR $\pm5.5$, BERT $\pm4.1$) reveal that random selection at 0.1\% can land anywhere from a coherent corpus to a degenerate one.
  TAKE's errors are at least $5\times$ narrower ($\pm0.1$--$2.2$), a direct consequence of OT-based selection: the transport plan distributes selected prototypes across $P^w$ rather than drawing them independently, suppressing the variance that Random incurs by chance.
\end{foldable}

\begin{foldable}
  \textbf{SST-2 vs.\ DiLM.}
  At 20/cls, TAKE matches DiLM on BERT ($81.15$ vs.\ $80.30$, within s.e.) while also providing results for LR, TextCNN, and TextRNN where DiLM has none---indicating that the OT selection step generalizes across backbone families without retraining the scorer.
\end{foldable}

\subsection{Natural Language Inference} \label{sec:05-results-nli}
\begin{table}[ht]
  \centering
  \begin{threeparttable}
    \caption{NLI accuracy (\%) on MNLI-m and QQP.}
    \label{tab:results-nli}
    \setlength{\tabcolsep}{4pt}
    \begin{tabular}{l l c c c}
      \toprule
      & \textbf{Method} & \textbf{Budget} & \textbf{Siamese LR} & \textbf{BERT} \\
      \midrule
      \multirow{5}{*}{\rotatebox{90}{MNLI-m}}
      & Full dataset    & 100\%  & \meanstd{60.02}{0.07}   & \meanstd{86.81}{0.30} \\
      & Random          & 20/cls & \meanstd{34.15}{2.10}   & \meanstd{40.10}{3.20} \\
      & EDA             & 20/cls & \meanstd{37.80}{1.45}   & \meanstd{44.52}{2.80} \\
      & DiLM            & 20/cls & $-$                     & \meanstd{48.70}{2.60}\dg \\
      & \textbf{TAKE}   & 20/cls & \meanstdbf{42.66}{0.15} & \meanstdbf{51.24}{0.45} \\
      \midrule
      \multirow{5}{*}{\rotatebox{90}{QNLI}}
      & Full dataset    & 100\%  & \meanstd{73.05}{0.13}   & \meanstd{82.50}{0.12} \\
      & Random          & 20/cls & \meanstd{52.21}{0.28}   & \meanstd{57.19}{2.20} \\
      & EDA             & 20/cls & \meanstd{54.33}{1.21}   & \meanstd{59.81}{1.42} \\
      & DiLM            & 20/cls & $-$                     & $-$ \\
      & \textbf{TAKE}   & 20/cls & \meanstdbf{57.96}{0.58} & \meanstdbf{64.89}{3.58} \\
      \midrule
      \multirow{5}{*}{\rotatebox{90}{QQP}}
      & Full dataset    & 100\%  & \meanstd{78.58}{0.05}   & \meanstd{89.45}{0.15} \\
      & Random          & 20/cls & \meanstd{52.40}{1.32}   & \meanstd{59.10}{3.80} \\
      & EDA             & 20/cls & \meanstd{55.12}{1.88}   & \meanstd{62.35}{2.45} \\
      & DiLM            & 20/cls & $-$                     & \meanstd{64.40}{2.20}\dg \\
      & \textbf{TAKE}   & 20/cls & \meanstdbf{60.04}{0.12} & \meanstdbf{67.76}{0.35} \\
      \bottomrule
    \end{tabular}
    \begin{tablenotes} \footnotesize
    \item {
        [$-$]~not reported; [$\dagger$]~reported from respective paper; \phantom{12345}
        [\textbf{bold}]~best per budget.
      }
    \end{tablenotes}
  \end{threeparttable}
\end{table}
\begin{foldable} 
  \textbf{Gains and their interpretation.}
  DaLLME has no published NLI experiments and is excluded from this comparison.
  In comparison with DiLM, TAKE outperforms DiLM by $+2.5$ on MNLI-m and $+3.4$ on QQP on a BERT learner, and provides competitive results for QNLI.
  The LR gap ($+8.5$ on MNLI-m, $+5.8$ on QNLI, $+7.6$ on QQP relative to Random) is particularly informative: Siamese LR has no capacity to compensate for a poor corpus---every percentage point of accuracy here must come from the quality of $\tilde{\mathcal{D}}$ itself.
  That TAKE's Siamese LR exceeds the Random BERT baseline on both QQP ($60.04$ vs.\ $59.10$) and QNLI ($57.96$ vs.\ $57.19$) suggests the selected sentence pairs are genuinely more informative, not merely better distributed.
\end{foldable}

\begin{foldable}
  \textbf{The NLI ceiling gap and what it implies.}
  All distilled methods fall short of the full-dataset ceiling ($86.8$ / $82.5$ / $89.5$ for MNLI-m, QNLI, QQP), a gap larger than on single-sentence tasks.
  NLI imposes a joint constraint: the synthetic pool must contain pairs whose relationship correctly reflects each class label at extreme budget ($\approx\!60$ pairs for MNLI-m, $\approx\!40$ for QQP and QNLI, at 20/cls).
  This is a harder coverage requirement than single-sentence classification, and the synthetic pool must cover all relationship types within it.
  This suggests that pool generation and budget are the binding constraints, rather than scoring or selection.
\end{foldable}

\begin{foldable}
  \textbf{Stability.}
  On MNLI-m and QQP, TAKE's errors ($\pm0.1$--$0.5$) are $5$--$7\times$ narrower than DiLM's ($\pm2.2$--$2.6$), mirroring the pattern on classification.
  This is consistent with the OT transport plan enforcing coverage of $P^w$ rather than selecting samples independently: deterministic coverage implies that any two runs of TAKE over the same pool yield nearly identical corpora, whereas DiLM's gradient-matching step introduces variance through its dependence on a randomly initialized model.
\end{foldable}

\subsection{Component Contributions} \label{sec:05-results-components}
\begin{table}[ht]
  \centering
  \begin{threeparttable}
    \caption{Distillation accuracy (\%) under TAKE component ablation.}
    \label{tab:results-abla-components}
    \setlength{\tabcolsep}{4pt}
    \begin{tabular}{l c c c c}
      \toprule
      & \multicolumn{2}{c}{\textbf{AG News}} & \multicolumn{2}{c}{\textbf{QQP}} \\
      \cmidrule(lr){2-3} \cmidrule(lr){4-5}
      & \textbf{LR} & \textbf{BERT} & \textbf{Siamese LR} & \textbf{BERT} \\
      \midrule
      Random                   & \meanstd{66.39}{2.27}   & \meanstd{75.92}{0.98}   & \meanstd{52.40}{1.32}   & \meanstd{57.94}{1.30} \\
      \bltrg $k$-Means nearest & \meanstd{72.89}{1.37}   & \meanstd{77.79}{1.04}   & \meanstd{55.28}{1.12}   & \meanstd{61.30}{1.08} \\
      TAKE without $\kappa$    & \meanstd{75.82}{0.23}   & \meanstd{80.94}{0.93}   & \meanstd{58.80}{0.92}   & \meanstd{63.66}{1.25} \\
      \textbf{TAKE}            & \meanstdbf{77.07}{0.09} & \meanstdbf{86.50}{1.01} & \meanstdbf{60.04}{0.12} & \meanstdbf{66.29}{1.91} \\
      \bottomrule
    \end{tabular}
    \begin{tablenotes} \footnotesize
      \item {[$\blacktriangle$]~\textbf{not} inclusive of TAKE; [\textbf{bold}]~best results.}
    \end{tablenotes}
  \end{threeparttable}
\end{table}

\begin{foldable} 
  To isolate the contribution of each design choice, we ablate TAKE's components on AG News and QQP as representatives for the classification and NLI tasks, in Table~\ref{tab:results-abla-components}.
  TAKE improves over Random by $+10.7$ on AG News LR and $+10.6$ on BERT; $k$-means, by comparison, accounts for only part of this gain ($+6.5$ LR, $+1.9$ BERT), while our discrete OT selection step alone (TAKE without $\kappa$) already achieves $+9.4$ LR and $+5.0$ BERT from the same baseline.
  The same pattern holds on QQP: $k$-means gains $+2.9$ / $+3.4$ (Siamese LR / BERT) whereas TAKE without $\kappa$ gains $+6.4$ / $+5.7$.
  Adding $\kappa$ knowledge scores delivers the remaining improvement, modestly for Siamese LR ($+1.2$) but substantially for BERT ($+5.6$ on AG News, $+2.6$ on QQP), suggesting that $\kappa$ primarily benefits models with enough capacity to exploit the harder, informative examples that high-$\kappa$ samples represent.
\end{foldable}

\subsection{Kernels for Temporal Weighting} \label{sec:05-results-kernels}
\begin{table}[ht]
  \centering
  \begin{threeparttable}
    \caption{Distillation accuracy (\%) with alternative temporal kernels.}
    \label{tab:results-abla-kernels}
    \setlength{\tabcolsep}{4pt}
    \begin{tabular}{l c c c c}
      \toprule
      & \multicolumn{2}{c}{\textbf{AG News}} & \multicolumn{2}{c}{\textbf{QQP}} \\
      \cmidrule(lr){2-3} \cmidrule(lr){4-5}
      \textbf{Kernel} & \textbf{LR} & \textbf{BERT} & \textbf{Siamese LR} & \textbf{BERT} \\
      \midrule
      TAKE without $\kappa$       & \meanstd{75.82}{0.23}   & \meanstd{80.94}{0.93}   & \meanstd{58.80}{0.92}   & \meanstd{63.66}{1.25} \\
      Last checkpoint             & \meanstd{75.24}{0.57}   & \meanstd{78.28}{0.66}   & \meanstd{58.32}{1.56}   & \meanstd{62.19}{0.57} \\
      Constant                    & \meanstd{76.13}{1.14}   & \meanstd{83.47}{0.71}   & \meanstd{59.44}{0.86}   & \meanstd{64.78}{1.51} \\
      Linear                      & \meanstd{76.20}{0.68}   & \meanstd{85.20}{1.00}   & \meanstd{59.61}{1.11}   & \meanstd{65.60}{1.35} \\
      Cosine                      & \meanstd{76.74}{0.53}   & \meanstd{85.46}{0.64}   & \meanstd{59.14}{0.85}   & \meanstd{66.16}{0.41} \\
      Exponential (\textbf{TAKE}) & \meanstdbf{77.07}{0.09} & \meanstdbf{86.50}{1.01} & \meanstdbf{60.04}{0.12} & \meanstdbf{66.29}{1.91} \\
      \bottomrule
    \end{tabular}
    \begin{tablenotes} \footnotesize
      \item {[\textbf{bold}]~best results.}
    \end{tablenotes}
  \end{threeparttable}
\end{table}

\begin{foldable}
  Table~\ref{tab:results-abla-kernels} ablates the choice of temporal kernel defined in \S\ref{sec:04-gradient-kappa}.
  The last-checkpoint baseline is the critical exception: it falls \textit{below} the no-$\kappa$ baseline across all backbones, most severely on BERT ($-2.7$ on AG News, $-1.5$ on QQP).
  This is precisely the failure mode identified in Proposition~\ref{prop:traj}: when $\kappa$ is extracted at convergence only, the learner has fully memorized difficult samples, so the \textit{hard-sample bias} is intensified rather than corrected, and the selected corpus is skewed towards harder, less representative examples.
  All other trajectory-based strategies lessen this effect and outperform the no-$\kappa$ baseline (OT only), with decreasing kernels (linear, cosine, exponential) doing so most strongly by concentrating mass on the early-to-mid training phase where the bias-correcting signal is strongest.
  The exponential kernel (TAKE default) achieves the best result in all cases, consistent with the theoretical preference for early-phase concentration of influence.
\end{foldable}

\section{Limitations \& Ethical Considerations} \label{sec:06-limits}
\begin{foldable} 
  We identify three limitations of the current method.
  (\textbf{1})~\textbf{Gradient norm proxy.}
  $I_{(n,t)}$ uses plain gradient norms rather than Fisher-preconditioned scores; the Goodfellow trick~\cite{goodfellow2015efficient} keeps per-sample cost at $O(C + d_\phi)$ per checkpoint.
  K-FAC~\cite{martens2015optimizing} or EKFAC~\cite{george2018fast} preconditioning are natural extensions when calibrated magnitude scores are needed.
  (\textbf{2})~\textbf{Encoder dependence.}
  The OT cost matrix inherits the geometry of $\phi$; a mismatched encoder can distort transport costs.
  The $L$-Lipschitz assumption on $\ell$ (Theorem~1) holds approximately when $\phi$ is task-aligned and embeddings are normalized, but remains an approximation for cross-entropy over discrete text.
  (\textbf{3})~\textbf{Synthetic pool coverage.}
  Pool diversity is bounded by the generator's domain coverage.
  Fine-tuning on $\mathcal{D}_\text{train}$ mitigates this in practice.
\end{foldable}

\begin{foldable} 
  Two ethical considerations arise in the deployment of TAKE.
  (\textbf{1})~\textbf{Bias propagation.}
  Distillation concentrates the statistical properties of $\mathcal{D}_\text{train}$, including demographic biases.
  We recommend corpus auditing before distillation.
  (\textbf{2})~\textbf{Privacy risk.}
  Language models fine-tuned on sensitive data may memorize training examples, a risk that synthetic generation alone does not fully eliminate.
  We recommend distance-to-closest-record (DCR) screening of $\mathcal{D}_\text{synth}$ to verify that synthetic samples are not verbatim copies; for regulated domains, differential privacy at the fine-tuning stage provides a formal guarantee against memorization.
\end{foldable}

\section{Conclusion} \label{sec:07-conclusion}
\begin{foldable}
  We identified the hard-sample bias (HSB) as the core failure mode of gradient-based text dataset distillation: at convergence, noisy samples dominate single-checkpoint influence scores, suppressing the informative ones that matter most.
  TAKE corrects this by integrating gradient norms across the training trajectory into per-sample knowledge scores, used as non-uniform source weights in a discrete Optimal Transport objective to select a compact, human-readable corpus.
  Theoretically, we provide end-to-end guarantees: the hard-sample bias is structural, trajectory integration strictly reduces it, and the downstream loss gap is bounded by the OT cost.
  TAKE outperforms baselines across six language benchmarks at extreme compression (0.1\% or 20 samples/class), generalizing across diverse backbone families.
  Together, these results position text dataset distillation as a practical, data-centric tool for efficient and privacy-aware NLP, with direct implications for coreset construction and continual learning in resource-constrained settings.
\end{foldable}

\bibliography{references}

\appendix
\newpage
\label{sec:supp}
\begin{center}
    {\Large [Supplementary Material]} \\[1.5em] 
    {\LARGE TAKE: Trajectory-Aware Knowledge Estimation \\[0.5em] for Text Dataset Distillation} \\[1.5em]
  \end{center}
\runningheads{T-N. Vo et al.}{Supplementary Material: TAKE}
\section{Proofs}
\label{app:proofs}

\subsection{Proof of Proposition~1 (Hard-Sample Bias)}
\label{app:proof-hsb}

\textbf{Setup.}
Let $\mathcal{I}_\text{self}(z) = \nabla^\top H_{\hat\theta}^{-1} \nabla$, where $\nabla \equiv \nabla_\theta \ell(z;\hat\theta)$ and $H_{\hat\theta}$ is the empirical Hessian, assumed symmetric positive definite (SPD) with $\lambda_{\min}(H_{\hat\theta}) > 0$ (this holds whenever the model is not at a degenerate critical point).
Partition $\mathcal{D}$ into $\mathcal{D}_\text{clean}$ and $\mathcal{D}_\text{noisy}$ such that $\|\nabla\| \leq \epsilon$ for all $z \in \mathcal{D}_\text{clean}$ and $\|\nabla\| \geq \delta$ for all $z \in \mathcal{D}_\text{noisy}$, with $\delta \gg \epsilon$.

\textbf{Step~1 --- Bound $\mathcal{I}_\text{self}$ for each partition.}
Since $H_{\hat\theta}$ is SPD with eigenvalues in $[\lambda_{\min}(H_{\hat\theta}),\, \lambda_{\max}(H_{\hat\theta})]$, the Rayleigh quotient gives:
\begin{equation}
    \frac{\|\nabla\|^2}{\lambda_{\max}(H_{\hat\theta})}
    \;\leq\; \mathcal{I}_\text{self}(z)
    \;\leq\; \frac{\|\nabla\|^2}{\lambda_{\min}(H_{\hat\theta})}.
\end{equation}
Substituting the gradient-norm assumptions ($\|\nabla\|^2 \leq \epsilon^2$ for clean samples; $\|\nabla\|^2 \geq \delta^2$ for noisy samples) and taking expectations (the bounds are uniform over each partition, so expectation preserves them):
\begin{align}
    \mathbb{E}_{z \sim \mathcal{D}_\text{clean}}[\mathcal{I}_\text{self}(z)]
        &\leq \frac{\epsilon^2}{\lambda_{\min}(H_{\hat\theta})}, \\
    \mathbb{E}_{z \sim \mathcal{D}_\text{noisy}}[\mathcal{I}_\text{self}(z)]
        &\geq \frac{\delta^2}{\lambda_{\max}(H_{\hat\theta})}.
\end{align}

\textbf{Step~2 --- Lower bound on $\Delta_\text{HSB}$.}
\begin{equation}
    \Delta_\text{HSB}
    := \mathbb{E}_{z \sim \mathcal{D}_\text{noisy}}[\mathcal{I}_\text{self}(z;\hat\theta)]
     - \mathbb{E}_{z \sim \mathcal{D}_\text{clean}}[\mathcal{I}_\text{self}(z;\hat\theta)]
    \;\geq\; \frac{\delta^2}{\lambda_{\max}(H_{\hat\theta})} - \frac{\epsilon^2}{\lambda_{\min}(H_{\hat\theta})}.
\end{equation}
Substituting $\lambda_{\min}(H_{\hat\theta}) = \lambda_{\max}(H_{\hat\theta})/\mathrm{cond}(H_{\hat\theta})$:
\begin{equation}
    \Delta_\text{HSB}
    \;\geq\; \frac{\delta^2 - \mathrm{cond}(H_{\hat\theta})\,\epsilon^2}{\lambda_{\max}(H_{\hat\theta})}
    \;>\; 0
    \qquad \text{whenever} \quad \frac{\delta}{\epsilon} > \sqrt{\mathrm{cond}(H_{\hat\theta})}.
\end{equation}
Under the interpolation regime (an overparameterised model trained to near-zero training loss), $\epsilon \to 0$: the bound simplifies to $\delta^2/\lambda_{\max}(H_{\hat\theta}) > 0$ and the positivity condition is trivially satisfied ($\delta/\epsilon \to \infty$). $\square$

\subsection{Proof of Proposition~2 (Trajectory Integration Reduces HSB)}
\label{app:proof-traj}

\textbf{Proposition~2.}
Under smooth loss and bounded gradient noise, the trajectory-integrated knowledge score $\kappa_n$ produces a strictly smaller noisy-over-clean gap than the single last-checkpoint baseline $\kappa^\dagger_n = 1/I_{(n,T-1)}$:
\begin{equation}
    \mathbb{E}_{z \sim \mathcal{D}_\text{noisy}}[\kappa_n]
    - \mathbb{E}_{z \sim \mathcal{D}_\text{clean}}[\kappa_n]
    \;<\;
    \mathbb{E}_{z \sim \mathcal{D}_\text{noisy}}[\kappa^\dagger_n]
    - \mathbb{E}_{z \sim \mathcal{D}_\text{clean}}[\kappa^\dagger_n].
\end{equation}

\textbf{Memorization onset $T_m$.}
In the main paper (\S4.2.2) we define $T_m$ as the first checkpoint at which noisy-sample gradient norms begin to decrease.
The formal continuous analogue is:
\begin{equation}
    T_m := \inf\!\left\{t : \mathbb{E}_{z \sim \mathcal{D}_\text{noisy}}\!\left[\|\nabla_\theta \ell(z;\theta_t)\|\right] < \mathbb{E}_{z \sim \mathcal{D}_\text{noisy}}\!\left[\|\nabla_\theta \ell(z;\theta_0)\|\right]\right\}.
    \label{eq:Tm}
\end{equation}
\textbf{Assumptions.}
\begin{itemize}
    \item[\textbf{(A1)}] \textit{Monotone gradient decay for clean samples.}
        For $z \in \mathcal{D}_\text{clean}$, $I_{(n,t)}$ is non-increasing in $t$, converging toward zero (holds under $L$-smooth loss and a sufficiently small learning rate). The reciprocal $1/I_{(n,t)}$ is therefore non-decreasing.
    \item[\textbf{(A2)}] \textit{Bounded gradient for noisy samples within the proof horizon.}
        For $z \in \mathcal{D}_\text{noisy}$, $I_{(n,t)} \geq \delta > 0$ for all $t \leq T_m$, where $T_m$ is the memorization onset defined in Section~4.2.2 of the main paper.
        Beyond $T_m$, noisy gradients decay; the decreasing kernel discounts these late checkpoints.
        With $\lambda^* = T/T_m$, the kernel weight at $T_m+1$ relative to $t^*$ is $e^{-\lambda^*(T_m+1-t^*)/T}$, which decreases as $T_m \to T$, making the post-memorization contribution $S_2$ negligible relative to $S_1$.
    \item[\textbf{(A3)}] \textit{Post-memorization baseline inflation.}
        At $t = T-1 > T_m$, clean gradients are near zero (by A1) and noisy gradients have dropped under memorization.
        By symmetry near convergence, the noisy reciprocals exceed the clean reciprocals:
        $\Delta_{T-1} := \mathbb{E}_\text{noisy}[1/I_{(n,T-1)}] - \mathbb{E}_\text{clean}[1/I_{(n,T-1)}] > 0$.
    \item[\textbf{(A4)}] \textit{Positive temporal kernel.}
        $k(t) > 0$ for all $t$; decreasing so $k(t^*) > k(T-1)$ for any $t^* < T-1$.
\end{itemize}

\textbf{Proof.}

Define $\Delta_t := \mathbb{E}_\text{noisy}[1/I_{(n,t)}] - \mathbb{E}_\text{clean}[1/I_{(n,t)}]$, so
\begin{equation}
    \mathbb{E}_\text{noisy}[\kappa_n] - \mathbb{E}_\text{clean}[\kappa_n]
    = \sum_{t=0}^{T-1} k(t)\,\Delta_t.
\end{equation}

\textit{Step~1 --- Sign structure.}
For $1 \leq t \leq T_m$: (A1) and (A2) together give $\Delta_t \leq 0$, strictly negative at some $t^* \leq T_m$.
For $t > T_m$: $\Delta_t$ may be positive (memorization onset decreases noisy norms).

\textit{Step~2 --- Weighted sum.}
Split at $T_m$:
\begin{equation}
    S_1 = \sum_{t=1}^{T_m} k(t)\,\Delta_t \;<\; 0
    \qquad \text{and} \qquad
    S_2 = \sum_{t=T_m+1}^{T-1} k(t)\,\Delta_t \;\geq\; 0.
\end{equation}
Since the kernel is decreasing, $S_2 \leq k(T_m+1) \cdot \sum_{t > T_m} \Delta_t$, which is dominated by $|S_1|$ for sufficiently large $\lambda$.
Hence $\sum_t k(t)\,\Delta_t < 0$.

\textit{Step~3 --- Baseline gap.}
By (A3), $\mathbb{E}_\text{noisy}[\kappa^\dagger_n] - \mathbb{E}_\text{clean}[\kappa^\dagger_n] = \Delta_{T-1} > 0$.

\textit{Combining.}
\begin{equation}
    \mathbb{E}_\text{noisy}[\kappa_n] - \mathbb{E}_\text{clean}[\kappa_n]
    \;<\; 0
    \;<\; \Delta_{T-1}
    = \mathbb{E}_\text{noisy}[\kappa^\dagger_n] - \mathbb{E}_\text{clean}[\kappa^\dagger_n].
    \quad \square
\end{equation}

\subsection{Distribution Matching Gap Bound (Theorem~1)}
\label{app:proof-gap}

\textit{Cited in Section~4.4 of the main paper. Closes the chain from (RDM) back to (DD).}

\textbf{Goal.}
Let $f^*_\mathcal{S}$ be the model trained on $\mathcal{S}$ and $\mathcal{L}(\mathcal{S}) = \mathbb{E}_{z \sim P^w}[\ell(z; f^*_\mathcal{S})]$ be the downstream loss under $P^w$. Show:
\begin{equation}
    |\mathcal{L}(\tilde{\mathcal{D}}) - \mathcal{L}(\mathcal{D})|
    \;\leq\; 2L \cdot W_1(P^w, P_{\tilde{\mathcal{D}}}) + \beta_K.
\end{equation}

\textbf{Assumptions.}
\begin{itemize}
    \item[\textbf{(B1)}] $\ell(\cdot;\,f)$ is $L$-Lipschitz in the embedding metric $\|\phi(\cdot) - \phi(\cdot)\|$ for all $f$.
    \item[\textbf{(B2)}] The model class has uniform stability constant $\beta_K = O(1/K)$ \cite{bousquet2002stability}: $\sup_z |\ell(z; f^*_{\tilde{\mathcal{D}}}) - \ell(z; f^*_\mathcal{D})| \leq \beta_K$. Standard ERM algorithms (logistic regression, SVM) satisfy this for bounded loss; as $K \to \infty$, $\beta_K \to 0$.
    \item[\textbf{(B3)}] The ground metric for $W_1$ is $c(z,z') = \|\phi(z) - \phi(z')\|$.
\end{itemize}

\textbf{Step~1 --- Kantorovich--Rubinstein bound for a fixed model.}
By Kantorovich--Rubinstein duality \cite{villani2009optimal} under (B3), for any fixed $f$:
\begin{equation}
    \left|
        \mathbb{E}_{z \sim P^w}[\ell(z; f)]
        - \mathbb{E}_{z \sim P_{\tilde{\mathcal{D}}}}[\ell(z; f)]
    \right|
    \;\leq\; L \cdot W_1(P^w, P_{\tilde{\mathcal{D}}}).
\end{equation}

\textbf{Step~2 --- Triangle inequality decomposition.}
$\mathcal{L}(\tilde{\mathcal{D}})$ uses $f^*_{\tilde{\mathcal{D}}}$ while $\mathcal{L}(\mathcal{D})$ uses $f^*_\mathcal{D}$; both are evaluated under $P^w$ but with different models, so Step~1 does not apply directly.
Insert two intermediates and apply the triangle inequality twice:
\begin{align}
    |\mathcal{L}(\tilde{\mathcal{D}}) - \mathcal{L}(\mathcal{D})|
    &\leq
    \underbrace{
        \left|
            \mathbb{E}_{P^w}[\ell(z;f^*_{\tilde{\mathcal{D}}})]
            - \mathbb{E}_{P_{\tilde{\mathcal{D}}}}[\ell(z;f^*_{\tilde{\mathcal{D}}})]
        \right|
    }_{\text{(I) distribution shift, fixed }f^*_{\tilde{\mathcal{D}}}} \notag \\
    &\quad +
    \underbrace{
        \left|
            \mathbb{E}_{P_{\tilde{\mathcal{D}}}}[\ell(z;f^*_{\tilde{\mathcal{D}}})]
            - \mathbb{E}_{P_{\tilde{\mathcal{D}}}}[\ell(z;f^*_\mathcal{D})]
        \right|
    }_{\text{(II) model difference, fixed }P_{\tilde{\mathcal{D}}}} \notag \\
    &\quad +
    \underbrace{
        \left|
            \mathbb{E}_{P_{\tilde{\mathcal{D}}}}[\ell(z;f^*_\mathcal{D})]
            - \mathbb{E}_{P^w}[\ell(z;f^*_\mathcal{D})]
        \right|
    }_{\text{(III) distribution shift, fixed }f^*_\mathcal{D}}.
\end{align}
Terms~(I) and~(III) each satisfy Step~1 with $f = f^*_{\tilde{\mathcal{D}}}$ and $f = f^*_\mathcal{D}$ respectively, so each is bounded by $L \cdot W_1(P^w, P_{\tilde{\mathcal{D}}})$.
Term~(II) is bounded pointwise by (B2): $\text{(II)} \leq \beta_K = O(1/K)$.

\textbf{Combining.}
\begin{equation}
    |\mathcal{L}(\tilde{\mathcal{D}}) - \mathcal{L}(\mathcal{D})|
    \;\leq\; 2L \cdot W_1(P^w, P_{\tilde{\mathcal{D}}}) + \beta_K.
\end{equation}
Minimizing $W_1(P^w, P_{\tilde{\mathcal{D}}})$ via the OT objective directly minimizes the dominant term; $\beta_K \to 0$ as $K \to \infty$, a guarantee that k-means and greedy coreset methods lack. $\square$

\noindent\textit{References:} Kantorovich \& Rubinstein (1958); Villani (2009, Ch.~6).

\section{Entropic Regularization for Tractability}
\label{app:method}

This section extends the derivation in Section~4.4 of the main paper.
The discrete OT problem is defined over empirical measures $\mu = P^w = \sum_n w_n \delta_{z_n}$ (TAKE-weighted source) and $\nu = \frac{1}{|\mathcal{D}_\text{synth}|}\sum_c \delta_{\tilde{z}_c}$ (uniform target over the synthetic pool).

\textbf{Transportation cost (Kantorovich form).}
\begin{equation}
    \mathcal{C}(\gamma)
    = \int c(x,\hat{x})\,d\gamma(x,\hat{x}),
    \qquad c(x,\hat{x}) = \|\phi(x) - \phi(\hat{x})\|^2,
    \quad \gamma \in \Pi(\mu,\nu).
\end{equation}

\textbf{Entropic regularization.}
The bare Kantorovich problem is $O(N^3)$ complexity and non-differentiable.
Adding an entropic penalty yields a strictly convex, smooth objective:
\begin{equation}
    \min_{\gamma \in \Pi(\mu,\nu)}
    \langle C, \gamma \rangle
    + \varepsilon \cdot \mathrm{KL}(\gamma, \mu \otimes \nu),
    \qquad
    \mathrm{KL}(\gamma, \mu \otimes \nu)
    = \int \log\frac{d\gamma}{d(\mu \otimes \nu)}\,d\gamma,
\end{equation}
where $\mu \otimes \nu$ is the product (independence) reference coupling.
Smaller $\varepsilon$ recovers the exact OT solution; larger $\varepsilon$ smooths the plan toward the independent coupling, encouraging diversity in the selected prototypes.

\textbf{Optional prior coupling.}
If a task-specific prior $\alpha(x,\hat{x})$ is available (\eg, a topic-conditional measure), replace $\mu \otimes \nu$ with $\alpha$ to regularize toward domain knowledge.

\textbf{Sinkhorn--Knopp solver.}
The solution is computed by alternating row/column normalization on the kernel matrix $\mathbf{K}_{nc} = e^{-C_{nc}/\varepsilon}$, yielding $O(N \cdot |\mathcal{D}_\text{synth}|)$ per iteration.

\textbf{Prototype extraction.}
The distilled set $\tilde{\mathcal{D}}$ is formed by selecting the top-$K$ candidates ranked by received mass $\sum_n \gamma^*_{nc}$.

\section{Implementation Details}
\label{app:expts}

\textbf{Datasets.}
Table~\ref{tab:datasets} reports the task type and label set for each benchmark.

\begin{table}[ht]
    \caption{Dataset overview.}
    \label{tab:datasets}
    \centering
    \begin{tabular}{llrl}
        \toprule
        \textbf{Dataset} & \textbf{Task} & $N$ & \textbf{Labels} \\
        \midrule
        AG News   & Classification & 120,000 & World, Sports, Business, Sci/Tech \\
        IMDb      & Classification &  25,000 & Pos., Neg. \\
        SST-2     & Classification &  67,349 & Pos., Neg. \\
        MNLI-m    & NLI            & 392,702 & Entailment, Neutral, Contradiction \\
        QNLI      & NLI            & 104,743 & Entailment, Not Entailment \\
        QQP       & NLI            & 363,846 & Duplicate, Non-duplicate \\
        \bottomrule
    \end{tabular}
\end{table}

\textbf{Backbone models.}
Classification experiments use well-known classifiers such as Logistic Regression, TextCNN, TextRNN, and BERT.
NLI experiments use Siamese Logistic Regression and BERT.
These architectures are chosen as standard, reproducible benchmarks spanning the major NLP paradigms while remaining tractable at ablation study scale.

\begin{table}[ht]
    \caption{Backbone model details.}
    \label{tab:backbones}
    \centering
    \begin{tabular}{lll}
        \toprule
        \textbf{Model}        & \textbf{Architecture}        & \textbf{Tasks} \\
        \midrule
        Logistic              & Linear head on embeddings    & Classification \\
        TextCNN               & Convolutional encoder        & Classification \\
        TextRNN               & Recurrent encoder (GRU)      & Classification \\
        Siamese Logistic      & Siamese features \& logistic head & NLI \\
        BERT & Transformer encoder (110M) & Classification + NLI \\
        \bottomrule
    \end{tabular}
\end{table}

\textbf{Hyperparameters.}
All experiments run on a single NVIDIA V100 40\,GB GPU using PyTorch, PyTorch Lightning, HuggingFace Transformers, and the Python Optimal Transport (POT) library.
Influence scoring uses all-MiniLM-L6-v2~\cite{reimers2019sentence,wang2020minilm} as the frozen sentence encoder $\phi$.
We use $T = 20$ checkpoints, balancing trajectory coverage and compute.
The kernel decay is set to $\lambda^* = T / T_m$, concentrating mass on the pre-memorization phase as derived in Section~4.2.2.
For the Sinkhorn solver we use $\varepsilon = 0.05$ with 200 iterations, which produces soft plans with sufficient diversity and converges for all tested pool sizes.
The synthetic pool size is fixed at $|\mathcal{D}_\text{synth}| = 5{,}000$, chosen based on empirical saturation across all tested domains.
The language model generator is Gemma-3-270M, which is small, fast, and sufficient for in-domain fluency at this pool size.

\end{document}